\documentclass[final, journal, 11pt, letterpaper]{IEEEtran}

\usepackage{epsf,epsfig,amsmath,amssymb,subfigure,algorithm,algorithmic,url,amsthm}

\usepackage{subfigure}

\usepackage{amsmath}
\usepackage{amsfonts}
\usepackage{amssymb}

\usepackage{algorithmic}
\usepackage{algorithm}

\usepackage{graphicx}
\usepackage{cite}
\usepackage{citesort}
\usepackage{color}

\newtheorem {theorem} {Theorem}
\newtheorem {lemma} {Lemma}

\newtheorem {definition} {Definition}

\newcommand{\trans}[1]{{#1}^{\ensuremath{\mathsf{T}}}}

\begin{document}

\title{Learning joint intensity-depth sparse representations\thanks{I. To\v{s}i\'c is with Ricoh Innovations, Corp., Menlo Park, USA, email: ivana@ric.ricoh.com. This work has been performed while she was with the Helen Wills Neuroscience Institute, University of California, Berkeley, USA. 
S. Drewes is with T-Systems International GmbH, Darmstadt, Germany, sarah.drewes@t-systems.com. She performed this work while she was with the Department of Industrial Engineering and Operations Research at University of California, Berkeley.}\thanks{This work has been supported by the Swiss National Science Foundation under the fellowship PA00P2-134159 awarded to I. To\v{s}i\'c.}}
\author{Ivana To\v{s}i\'c and Sarah Drewes}


\maketitle

\begin{abstract}
This paper presents a method for learning overcomplete dictionaries of atoms composed of two modalities that describe a 3D scene: image intensity and scene depth. We propose a novel Joint Basis Pursuit (JBP) algorithm that finds related sparse features in two modalities using conic programming and we integrate it into a two-step dictionary learning algorithm. JBP differs from related convex algorithms because it finds joint sparsity models with different atoms and different coefficient values for intensity and depth. This is crucial for recovering generative models where the same sparse underlying causes (3D features) give rise to different signals (intensity and depth). We give a  bound for recovery error of sparse coefficients obtained by JBP, and show numerically that JBP is superior to the Group Lasso (GL) algorithm. When applied to the Middlebury depth-intensity database, our learning algorithm converges to a set of related features, such as pairs of depth and intensity edges or image textures and depth slants. Finally, we show that JBP (with the learned dictionary) outperforms both GL and Total Variation (TV) on depth inpainting for time-of-flight 3D data.
\end{abstract}

\section{Introduction}

Hybrid image-depth sensors have recently gained a lot of popularity in many vision applications. Time of flight cameras~\cite{hagebeuker3d,swissranger} provide real-time depth maps at moderate spatial resolutions, aligned with the image data of the same scene. Microsoft Kinect~\cite{kinect} also provides real-time depth maps that can be registered with color data in order to provide 3D scene representation. Since captured images and depth data are caused by the presence of same objects in a 3D scene, they represent two modalities of the same phenomena and are thus correlated. This correlation can be advantageously used for denoising corrupted or inpainting missing information in captured depth maps. Such algorithms are of significant importance to technologies relying on image-depth sensors for 3D scene reconstruction or visualization~\cite{kinect,Kubota:2007p1352}, where depth maps are usually noisy, unreliable or of poor spatial resolution. 


Solving inverse problems such as denoising or inpainting usually involves using prior information about data. Sparse priors over coefficients in learned linear generative models have been recently applied to these problems with large success~\cite{olshausen97sparse,Lewicki:2000p2757,Aharon06}. A similar approach has been proposed for learning sparse models of depth only, showing state-of-the-art performance in depth map denoising and offering a general tool for improving existing depth estimation algorithms~\cite{tosic2011learning}. However, learning sparse generative models for joint representation of depth and intensity images has not been addressed yet. Learning such models from natural 3D data is of great importance for many applications involving 3D scene reconstruction, representation and compression. 

This paper proposes a method for learning joint depth and intensity sparse generative models. Each of these two modalities is represented using overcomplete linear decompositions, resulting in two sets of coefficients. These two sets are coupled via a set of hidden variables, where each variable multiplies exactly one coefficient in each modality. Consequently, imposing a sparse prior on this set of coupling variables results in a common sparse support for intensity and depth. Each of these hidden variables can be interpreted as presence of a depth-intensity feature pair arising from the same underlying cause in a 3D scene. To infer these hidden variables under a sparse prior, we propose a convex, second order cone program named Joint Basis Pursuit (JBP). Compared to Group Lasso (GL)~\cite{Bakin99}, which is commonly used for coupling sparse variables, JBP gives significantly smaller coefficient recovery error. In addition, we bound theoretically this error by exploiting the restricted isometry property (RIP)~\cite{ca:ta:04} of the model. 
Finally, we propose an intensity-depth dictionary learning algorithm based on the new model and JBP. We show its superiority to GL in model recovery experiments using synthetic data, as well as  in inpainting experiments using real time-of-flight 3D data.

We first explain in Section~\ref{sec:badmodels} why existing models are not sufficient for intensity-depth representation. Section~\ref{sec:model} introduces the proposed intensity-depth generative model. Inference of its hidden variables is achieved via the new JBP algorithm presented in Section~\ref{sec:JBP}, while learning of model parameters is explained in Section~\ref{sec:learning}. Section~\ref{sec:sota} gives relations of the proposed JBP to prior art. Experimental results are presented in Section~\ref{sec:results}.


\section{Why aren't existing models enough?}\label{sec:badmodels}
To model the joint sparsity in intensity and depth, one might think that simple, existing models would suffice. For example, an intuitive approach would be to simply merge depth and  image pixels into one array of pixels. If we denote the vectorized form of the intensity image as $\mathbf{y}^I$ and depth image as $\mathbf{y}^D$, this "merged" model can be written as:
\[
\left[ \begin{array}{cc}
\mathbf{y}^I \\
\mathbf{y}^D
\end{array} \right]=
\left[ \begin{array}{cc}
\boldsymbol{\Phi}^I \\
\boldsymbol{\Phi}^D  \end{array} \right] \cdot
\mathbf{c}
\]
where intensity and depth are assumed to be sparse in dictionaries $\boldsymbol{\Phi}^I $, resp. $\boldsymbol{\Phi}^D$. The sparse vector $\mathbf{c}$ would then couple the sparse patterns in intensity and depth, i.e., couple intensity and depth atoms in pairs. However, since the vector of coefficients $\mathbf{c}$ is common, intensity and depth atoms within a pair will be multiplied with the same value. Let us now look at two simple synthetic examples of 3D scenes whose intensity and depth images are shown on Fig.~\ref{fig:examples}. The first example is a 3D edge and the second is a textured pattern on a slanted surface. These are two common intensity-depth features in real scenes. Since it has the flexibility of using different atoms for intensity and depth, the merged model will be able to represent both features. However, since the coefficients are common between intensity and depth, the variability in magnitude between intensity and depth would have to be represented by different atom pairs, leading to a combinatorial explosion in dictionary size. 

Another model that has been widely used in literature for representing correlated signals is the joint sparsity model, where signals share the same sparse support in $\boldsymbol{\Phi}$, but with different coefficients:
\[
\left[ \begin{array}{cc}
\mathbf{y}^I \\
\mathbf{y}^D
\end{array} \right]=
\boldsymbol{\Phi} \cdot
\left[ \begin{array}{cc}
\mathbf{a}\\
\mathbf{b}\end{array} \right].
\]
Therefore, the property of this model is that signals are represented using the same atoms multiplied by different coefficients. Obviously, the joint sparsity model would be able to represent the intensity-depth edge in Fig.~\ref{fig:examples} using a piecewise constant atom and different coefficients for intensity and depth. However, in the slanted texture example, because the depth image is linear and the intensity is a chirp, no atom can model both. The joint sparsity model would then have to decouple these two features in different atoms, which is suboptimal for representing slanted textures.

It becomes clear that we need a model that allows joint representation with different atoms and different coefficients, but with a common sparse support (the pattern of non-zero coefficients needs to be the same). We introduce such a model in the next section.

\begin{figure}[!tbp]
\begin{center}
\includegraphics[width=0.8\columnwidth]{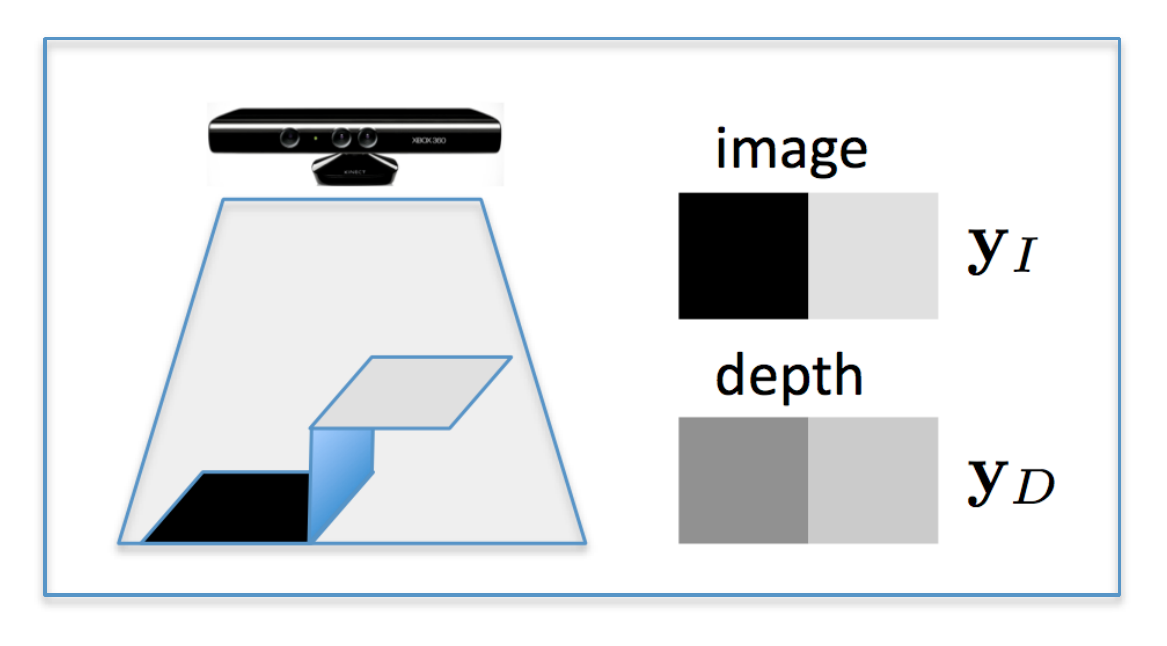}\\
\mbox{ (a) }\\
~\\
\includegraphics[width=0.8\columnwidth]{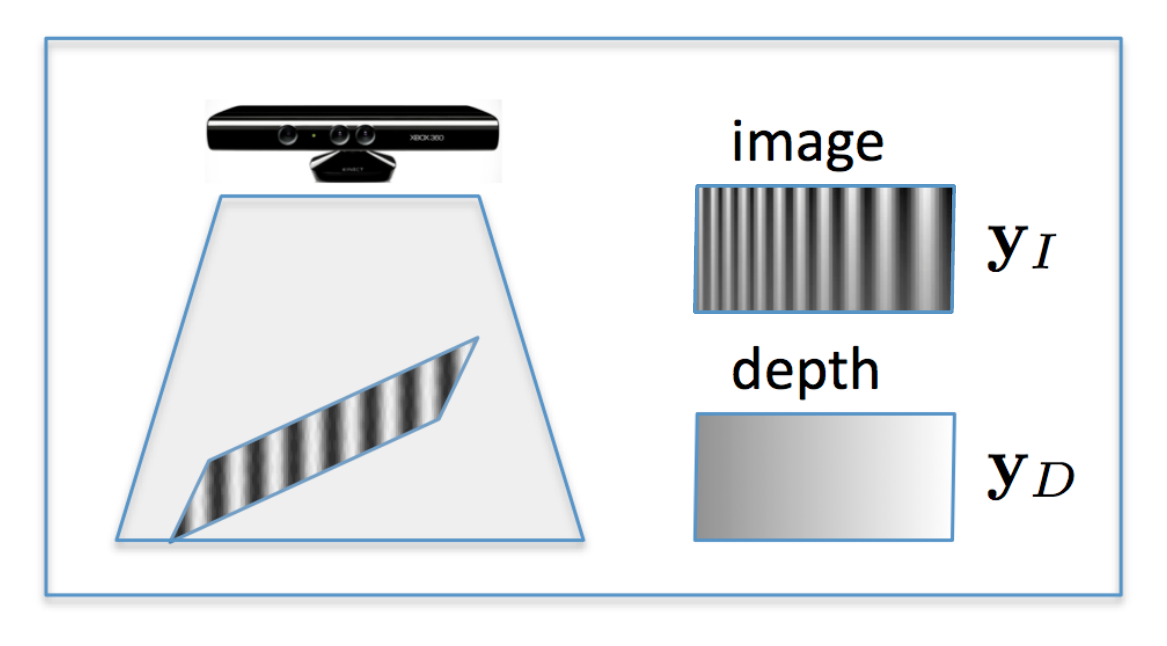}\\
\mbox{ (b)}\\
\end{center}
\caption{Examples of two typical image-depth features in 3D scenes. (a) Example 1: 3D edge, (b) Example 2: slanted texture.}
\label{fig:examples}
\end{figure}

\section{Intensity-depth generative model}\label{sec:model}

Let us first set the notation rules. Throughout the rest of the paper, vectors are denoted with bold lower case letters and matrices with bold upper case letters. Letters $I,D$ in superscripts refer to intensity and depth, respectively. Sets are represented with calligraphic fonts. Column-wise and row-wise concatenations of vectors $\mathbf{a}$ and $\mathbf{b}$ are denoted as $[\mathbf{a}\hspace{0.2cm} \mathbf{b}]$ and $[\mathbf{a}; \mathbf{b}]$, respectively. 

\begin{figure*}[!th]
\begin{center}
\includegraphics[width=0.6\textwidth]{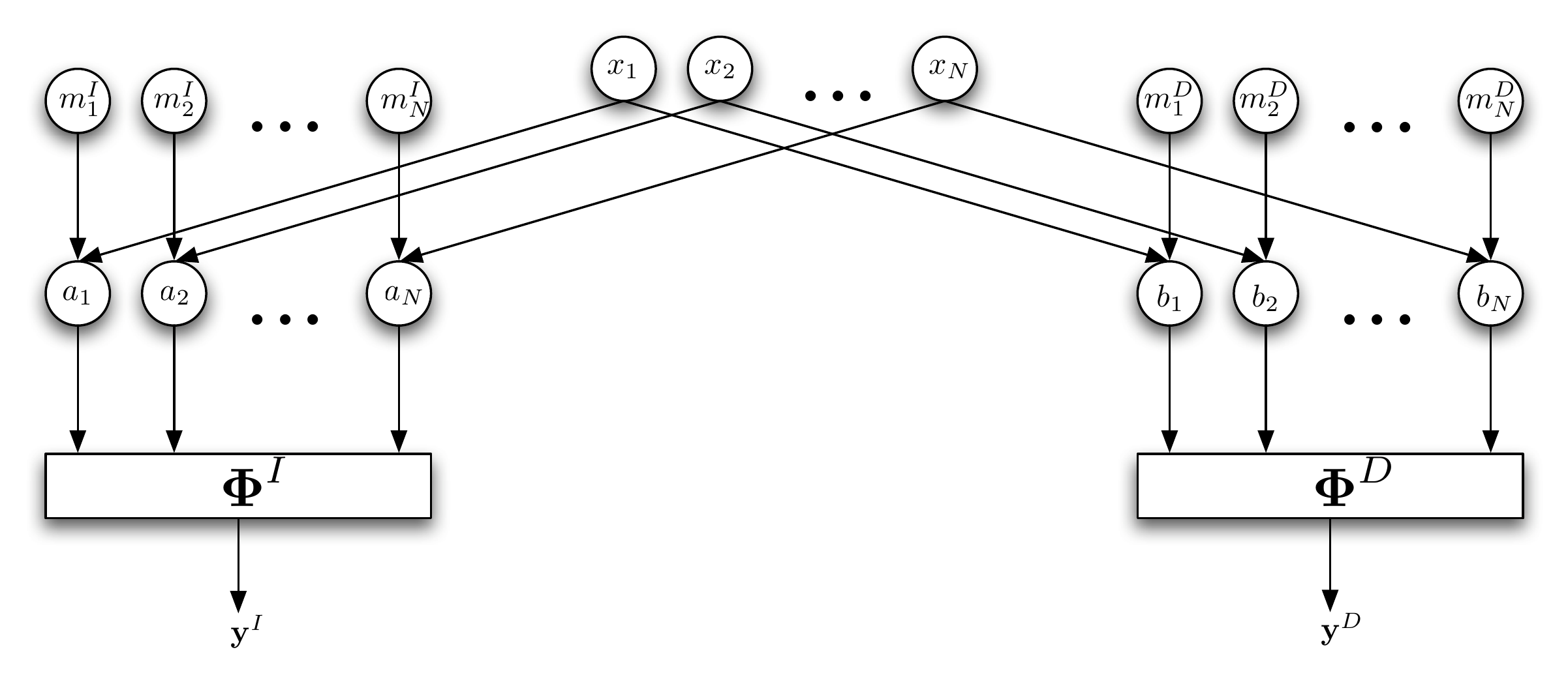}
\caption{Graphical representation of the proposed intensity-depth generative model.}
\label{fig:gm}
\end{center}
\end{figure*}

Graphical representation of the proposed joint depth-intensity generative model is shown in Fig.~\ref{fig:gm}. Intensity image $\mathbf{y}^I$ and depth image $\mathbf{y}^D$ (in vectorized forms) are assumed to be sparse in dictionaries $\boldsymbol{\Phi}^I $, resp. $\boldsymbol{\Phi}^D$, i.e., they are represented as linear combinations of dictionary atoms $\{ \boldsymbol{\phi}^I_i\}_{i \in \mathcal{I}}$ and $\{ \boldsymbol{\phi}^D_i\}_{i \in \mathcal{I}}$, resp. :
\begin{align}
\label{eq:images}
    \mathbf{y}^I&  = \boldsymbol{\Phi}^I \mathbf{a} +  \boldsymbol{\eta}^I = \sum_{i \in \mathcal{I}_0} \boldsymbol{\phi}^I_i a_i +  \boldsymbol{\eta}^I\nonumber\\
    \mathbf{y}^D&  = \boldsymbol{\Phi}^D \mathbf{b} +  \boldsymbol{\eta}^D = \sum_{i \in \mathcal{I}_0}  \boldsymbol{\phi}^D_i b_i +  \boldsymbol{\eta}^D,
\end{align}
where vectors $\mathbf{a}$ and $\mathbf{b}$ have a small number of non-zero elements and $ \boldsymbol{\eta}^I$ and $ \boldsymbol{\eta}^D$ represent noise vectors. $\mathcal{I}_0$ is the set of indexes identifying the columns (i.e., atoms) of $\boldsymbol{\Phi}^I $ and $\boldsymbol{\Phi}^D$ that participate in sparse representations of  $\mathbf{y}^I$ and $\mathbf{y}^D$. Its cardinality is much smaller than the dictionary size, hence $|\mathcal{I}_0| \ll |\mathcal{I}|$, where $\mathcal{I}=\{1,2,...,N\}$ denotes the index set of all atoms. This means that each image can be represented as a combination of few, representative features described by atoms, modulated by their respective coefficients. Because depth and intensity features correspond to two modalities arising from the same 3D features, we model the coupling between coefficients $a_i$ and $b_i$ through latent variables $x_i$ as:
\begin{align}\label{eq:coefs}
 a_i = m^I_i x_i; \hspace{0.5cm} b_i = m^D_i x_i, \hspace{0.4 cm} \forall i \in \mathcal{I},
\end{align}
where the variables $m_i^I, m_i^D$ represent the magnitudes of the sparse coefficients and $x_i$ represent the activity of these coefficients. Ideally, these variables should be binary, $0$ representing the absence and $1$ representing the presence of a depth-intensity feature pair. In that case $\sum_i x_i $ counts the number of non-zero such pairs.  However, inference of binary values represents a combinatorial optimization problem of high complexity which depends on dictionary properties and the permission of noise, cf. \cite{ca:ro:te:05}. We relax the problem by allowing $x_i$ to attain continuous values between $0$ and $1$, which has been proven to provide a very good approximation in a similar context, cf., e.g., \cite{do:04,jo:pf:07}.

By assuming that the vector $\mathbf{x}=\trans{(x_1,x_2,...,x_N)}$ is sparse, we assume that $\mathbf{y}^I$ and $\mathbf{y}^D$ are described by a small number of feature pairs $(\boldsymbol{\phi}^I_i,\boldsymbol{\phi}^D_i)$ that are either prominent in both modalities (both $m_i^I$ and $m_i^D$ are significant) or in only one modality (either $m_i^I$ or $m_i^D$ is significant). In these cases $x_i$ is non-zero, which leads to non-zero values for either $a_i$ or $b_i$, or both. If $x_i$ is zero, both $a_i$ and $b_i$ are also zero. Hence, the sparsity assumption on $\mathbf{x}$ enforces a compact description of both modalities by using simultaneously active coefficients. In addition, when such pairs cannot approximate both images, the model also allows only one coefficient within a pair to be non-zero. Therefore, the model represents intensity and depth using a small set of joint features and a small set of independent features. The main challenge is to simultaneously infer the latent variables $\mathbf{x}$, $\mathbf{m}^I=\trans{(m^I_1,m^I_2,...,m^I_N)}$ and $\mathbf{m}^D=\trans{(m^D_1,m^D_2,...,m^D_N)}$ under the sparsity assumption on $\mathbf{x}$. In the next section we propose a convex algorithm that solves this problem.

\section{Joint Basis Pursuit}\label{sec:JBP}
Let us re-write the intensity-depth generative model, including all unknown variables, in matrix notation as:
\[
\left[ \begin{array}{cc}
\mathbf{y}^I \\
\mathbf{y}^D
\end{array} \right]=
\left[ \begin{array}{c @{\hspace{0.4 cm}} c}
\boldsymbol{\Phi}^I  & 0 \\
0 & \boldsymbol{\Phi}^D  \end{array} \right] \cdot
\left[ \begin{array}{cc}
\mathbf{M}^I \\
\mathbf{M}^D
\end{array} \right]\cdot
\mathbf{x}+
\left[ \begin{array}{cc}
\boldsymbol{\eta}^I \\
\boldsymbol{\eta}^D
\end{array} \right],
\]
where $\mathbf{M}^I = \text{diag}{(m^I_1, m^I_2,..., m^I_N)}    $ and $\mathbf{M}^D= \text{diag}{(m^D_1, m^D_2,..., m^D_N)}$. Suppose first that we know dictionaries $\boldsymbol{\Phi}^I$ and $\boldsymbol{\Phi}^D$ and we want to find joint sparse representations of intensity and depth, i.e., to solve for variables $\mathbf{x}, \mathbf{m}^I, \mathbf{m}^D$. To do this, we formulate the following optimization problem:
\begin{align}
\text{OPT1}: \hspace{0.1 cm} &\min{\sum_i x_i}, \hspace{0.2 cm} \text{where}\hspace{0.1 cm} x_i \in [0,1],\hspace{0.1 cm} i=1,...,N\nonumber\\
\text{subject to:} \hspace{0.1 cm} &\| \mathbf{y}^I - \boldsymbol{\Phi}^I \mathbf{M}^I \mathbf{x} \|^2 \leq (\epsilon^I)^2 \label{o1:6}\\
&\| \mathbf{y}^D - \boldsymbol{\Phi}^D\mathbf{M}^D \mathbf{x} \|^2 \leq (\epsilon^D)^2 \label{o1:7}\\
& |m^I_i| \leq U^I \label{o1:8}\\
& |m^D_i| \leq U^D \label{o1:9}
\end{align}
where $\epsilon^I$, $\epsilon^D$ are allowed approximation errors and $U^I$ and $U^D$ are upper bounds on the magnitudes $\mathbf{m}^I$ and $\mathbf{m}^D$. In practice, the values of these upper bounds can be chosen as arbitrarily high finite values. This optimization problem is hard to solve using the above formulation, since the first two constraints are non-convex due to the terms $\mathbf{M}^I \mathbf{x}$ and $\mathbf{M}^D \mathbf{x}$ which are bilinear in the variables $\mathbf{x}$, $\mathbf{m}^I$ and $\mathbf{m}^D$. To overcome this issue, we transform it into an equivalent problem by introducing the change of variables given by Eqs.~(\ref{eq:coefs}) deriving:
\begin{align}
\text{OPT2}: \hspace{0.2 cm} &\min{\sum_i x_i}, \hspace{0.2 cm} \text{where}\hspace{0.1 cm} x_i \in [0,1],\hspace{0.1 cm} i=1,...,N\nonumber\\
\text{subject to:} \hspace{0.4 cm} &\| \mathbf{y}^I - \boldsymbol{\Phi}^I \mathbf{a} \|^2 \leq (\epsilon^I)^2 \label{o2:11}\\
&\| \mathbf{y}^D - \boldsymbol{\Phi}^D\mathbf{b} \|^2 \leq (\epsilon^D)^2 \label{o2:12}\\
& |a_i| \leq U^I x_i \label{o2:13}\\
& |b_i| \leq U^D x_i, \label{o2:14}
\end{align}
which is a convex optimization problem with linear and quadratic constraints that can be solved efficiently, i.e., in polynomial time, using log-barrier algorithms, cf. \cite{Tsuchiya98aconvergence,an:ro:te:03}.  A variety of free and commercial software packages are available like IBM ILOG CPLEX \cite{cplex}, that we use in our experiments. \\

The problems (OPT1) and (OPT2) are indeed equivalent using the variable transformation in Eqs.~(\ref{eq:coefs}) as follows. 
\begin{lemma}
For any optimal solution $(\mathbf{x}^*,\mathbf{a}^*, \mathbf{b}^*)$ of (OPT2), $\mathbf{x}^*$ is also an optimal solution to (OPT1) with corresponding matrices $ (\mathbf{M^I})^{*}$, $ (\mathbf{M^D})^{*}$ according to \eqref{eq:coefs}. \\ 
Also, any optimal solution $(\mathbf{x}^*, (\mathbf{M^I})^{*}, (\mathbf{M^D})^{*})$  of (OPT1) defines an optimal solution $(\mathbf{x}^*,\mathbf{a}^*, \mathbf{b}^*)$ to (OPT2) .
\end{lemma}
\begin{proof}
For any $(\mathbf{x}^*,\mathbf{a}^*, \mathbf{b}^*)$  and corresponding $ (\mathbf{M^I})^{*}$, $ (\mathbf{M^D})^{*}$ that satisfy Eqs.~(\ref{eq:coefs}), conditions \eqref{o2:11} and \eqref{o2:12} are equivalent to \eqref{o1:6} and \eqref{o1:7} by definition. Moreover, since $\mathbf{x}^*$ is nonnegative, conditions \eqref{o2:13} and \eqref{o2:14} are equivalent to \eqref{o1:8} and \eqref{o1:9}. Hence, any $\mathbf{x}^*$ that is optimal for (OPT2) with corresponding $(\mathbf{a}^*, \mathbf{b}^*)$ is optimal for (OPT1) with corresponding  $ (\mathbf{M^I})^{*}$, $ (\mathbf{M^D})^{*}$  and vice versa. 
\end{proof}

An immediate consequence of the form of the objective function and constraints in (OPT2) is that $\mathbf{x}^{*}$ is chosen such that \eqref{o2:13} and \eqref{o2:14} are both feasible and at least one of them is active. Formally, this is stated by the following lemma.
\begin{lemma}\label{lm:bound_activity}
For any optimal solution $(\mathbf{x}^*,\mathbf{a}^*, \mathbf{b}^*)$ of (OPT2), at least one of the constraints \eqref{o2:13} and \eqref{o2:14} is active for each component $i$, hence we have
\begin{equation}
\label{eq:max}
x_{i}^{*} = \max \{ \frac{|a^{*}_{i}|}{U^{I}}, \frac{|b^{*}_{i}|}{U^{D}}  \}, \hspace{0.5cm} \forall i=1,...,N.
\end{equation}

 \end{lemma}
\begin{proof}
Otherwise it would be a contradiction to the optimality of $\mathbf{x}^{*}$.
\end{proof}
~\\
In the following, we refer to the optimization problem (OPT2) as Joint Basis Pursuit (JBP), where $\mathbf{x}$ is the vector of joint (coupling) variables in the signal model. It is important to know the theoretical bounds on the norm of the difference between the solution $(\mathbf{a}^*, \mathbf{b}^*)$ found by JBP and the true coefficients $(\mathbf{a}, \mathbf{b})$ of the model (\ref{eq:images}). 

Based on the non-coupled case that is treated in \cite{ca:ro:te:05}, we develop bounds on the difference of the optimal solution of (OPT2) and a sparse signal to be recovered.
For this purpose, we assume that the matrix
\begin{equation}
\label{eq:matA}
\mathbf{A}:= \left[\begin{array}{cc} \boldsymbol{\Phi^{I}} & \mathbf{0} \\ \mathbf{0} & \boldsymbol{\Phi^{D}} \end{array}\right]
\end{equation}
satisfies the restricted isometry property with a constant $\delta_S$. This property of a linear system is defined as follows. Denote $\mathbf{A}_\mathcal{T}$, $\mathcal{T} \subset {1,...,n}$ as the $n\times |\mathcal{T}|$ submatrix obtained by extracting the columns of $\mathbf{A}$ corresponding to the indices in set $\mathcal{T}$, and $|\cdot |$ denotes the cardinality of the set. The S-restricted isometry constant $\delta_S$ is then defined as:
\begin{definition}~\cite{ca:ta:04}
\label{def:isometry}
The S-restricted isometry constant $\delta_S$ of $\mathbf{A}$ is the smallest quantity such that
\begin{equation}
\label{eq:RIP}
(1-\delta_S)\|\mathbf{s} \|_2^2 \leq \|\mathbf{A}_\mathcal{T} \mathbf{s}\|_2^2 \leq (1+\delta_S)\|\mathbf{s} \|_2^2
\end{equation}
for all subsets $\mathcal{T}$ with $|\mathcal{T} | \leq S$ and coefficient sequences $(s_j )$, $j\in \mathcal{T}$ . 
\end{definition}
When $\delta_S<<1$, this property requires that every set of columns with cardinality less than $S$ approximately behaves like an orthonormal system. It can thus be related to the maximal value of the inner product between any two columns in the matrix $\mathbf{A}$, usually called the coherence of the dictionary:
\begin{equation}
\label{eq:coherence}
\mu = \max_{i,j\neq i} | \langle \boldsymbol{\phi}_i,\boldsymbol{\phi}_j \rangle |,
\end{equation}
where $\boldsymbol{\phi}_i$ and $,\boldsymbol{\phi}_j$ are two different atoms in the dictionary (i.e., two columns of $\mathbf{A}$) and $\langle \cdot \rangle$ denotes the inner product. With this definition, it can be easily shown that $\delta_S = \mu(|\mathcal{T}|-1)$ satisfies the RIP inequality~(\ref{eq:RIP}).

Before we present the bound on the coefficient recovery error of JBP, let us first define some prerequisites. Assume we are given a pair of sparse signals $(\mathbf{y}^I, \mathbf{y}^D)$ as in Eq.~(\ref{eq:images}), with sparse coefficients $(\mathbf{a}^{0},\mathbf{b}^{0})$, which satisfy constraints \eqref{o2:11} and \eqref{o2:12}. Let $\mathcal{T}_{0}$ be the support of $\mathbf{x}^{0}$ which is at the same time the support of at least $\mathbf{a}^{0}$ or $\mathbf{b}^{0}$ and contains the support of the other one or it coincides with the support of both.  Without loss of generality, let us assume that 
\begin{equation}
\label{eq:f0}
\|\mathbf{y}^I \|_2=\|\mathbf{y}^D\|_2=:f_0,
\end{equation}
 which can be easily obtained by normalization. Assume also that the components of $\mathbf{a}^0$ and $\mathbf{b}^0$ satisfy the bound constraints\footnote{Although the assumption in Eq.~(\ref{eq:ass_coef}) does not hold in general, in practical applications using learned dictionaries we found that it is always satisfied. However, if one wants to use a bound that is surely satisfied, one should choose $U=f_0/\sigma_{min}$, where $\sigma_{min}$ is the smallest of all singular values of $\boldsymbol{\Phi}^I$ and $\boldsymbol{\Phi}^D$.}
\begin{align}\label{eq:ass_coef}
 |a^0_i| \leq f_0, \hspace{0.2cm} |b^{0}_i| \leq f_0, \; \forall i=1,...,N,
\end{align}
i.e., in the remainder of the paper we assume the same bounds on $a_i$ and $b_i$: $U^{I} = U^{D} = U = f_0$. It is also useful in practice to select the approximation error $\epsilon$ in terms of the fraction of the total signal energy, so we denote $\epsilon = \eta f_0$, where $0 \leq \eta < 1$.

Let further $\alpha_i$ denote the scale between the smaller and larger coefficient for each index $i$ within the sparse support set $\mathcal{T}_0$, i.e.:
\begin{equation}
\label{eq:alpha}
\alpha_i = \min \{ \frac{|a^0_i|}{|b^0_i|},\frac{|b^0_i|}{|a^0_i|} \}, \hspace{0.5cm} \forall i \in \mathcal{T}_0,
\end{equation}
and let $\gamma$ denote:
\begin{equation}
\label{eq:gamma}
\gamma = 1- \min_{i \in \mathcal{T}_0}  \alpha_i.
\end{equation}
Parameter $\gamma$ describes the level of similarity between the sparse coefficients in the two signals, which is decreasing with higher similarity. In the trivial case when $a^0_i = b_i^0$, $\forall i \in \mathcal{T}_0$ we have that $\gamma = 0$. In all other cases $\gamma \leq 1$. 

Let further $\mathbf{x}^{0}$ denote an auxiliary vector that satisfies
\[
\max\{|a^{0}_i|,|b^{0}_i|\} =  U x^{0}_{i}, \hspace{0.5cm} \forall i \in \mathcal{T}_0
\]
namely $(\mathbf{x}^{0}, \mathbf{a}^{0}, \mathbf{b}^{0})$ is  a feasible solution to (OPT2), where $\mathbf{x}^{0}$ is chosen such that \eqref{o2:13} and \eqref{o2:14} are both feasible and (at least) one of them is active.

Finally, let $(\mathbf{x}^*, \mathbf{a}^*, \mathbf{b}^*)$ be an optimal solution to (OPT2). Then we have the following worst case bound on the distance of these.
\begin{theorem}\label{thm:bound}
Let $(\mathbf{a}^{0}, \mathbf{b}^{0})$ and $(\mathbf{a}^*, \mathbf{b}^*)$ as defined above and choose $U=f_0$ with $f_{0}$ from \eqref{eq:f0} and $\epsilon^I = \epsilon^D = \eta f_0$, where $0 \leq \eta < 1$. Then
\begin{equation}\label{eq:bound}
\| [\mathbf{a}^{0};\mathbf{b}^{0}]  - [\mathbf{a}^*; \mathbf{b}^*] \|_{2}^2 \leq \left[ \frac{|\mathcal{T}_0|}{M}(C+\gamma\sqrt{|\mathcal{T}_0|})^2 + C^2 \right] f_0^2
\end{equation}
holds for a constant $C$ that depends on the signal model parameter $\gamma$, the sparse support size $|\mathcal{T}_0|$ and the approximation parameter $\eta$, and where the $M$-restricted isometry property is satisfied for the linear system, cf. Def. \ref{def:isometry}. In particular, we have:
\begin{equation}
\label{eq:constC}
C = \frac{4\eta \sqrt{M} + \gamma |\mathcal{T}_0| \sqrt{1+\delta_M}}{\sqrt{M(1-\delta_{M+|\mathcal{T}_0|})}-\sqrt{|\mathcal{T}_0|(1+\delta_M)}}.
\end{equation}
\end{theorem}
~\\
The proof of this Theorem is given in Appendix~\ref{app:proof}. 



\section{Intensity-depth dictionary learning}\label{sec:learning}
In the previous section we have shown how to find sparse coefficients in the joint depth-intensity generative model, assuming that the model parameters, i.e., dictionaries $\boldsymbol{\Phi}^I$ and $\boldsymbol{\Phi}^D$ are given. Since we do not have those parameters in general, we propose to learn them from a large database of intensity-depth image examples. Dictionary learning for sparse approximation has been a topic of intensive research in the last couple of years. Almost all existing algorithms are based on Expectation-Maximization, i.e., they are iterative algorithms that consist of two steps: 1) inference of sparse coefficients for a large set of signal examples while keeping the dictionary parameters fixed, and 2) dictionary optimization to minimize the reconstruction error while keeping the coefficients fixed. We follow the same approach here, using JBP in the first step and then conjugate gradient in the second step. Once JBP finds the sparse coefficients $(\mathbf{a}^*, \mathbf{b}^*)$ and the coupling variables $\mathbf{x}$, optimization of $\boldsymbol{\Phi}^I$ and $\boldsymbol{\Phi}^D$ becomes decoupled. Therefore, in the learning step we independently optimize the following objectives:
\begin{align}
\label{eq:learning}
    (\boldsymbol{\Phi}^I)^* &= \min_{\boldsymbol{\Phi}^I}  \| \mathbf{Y}^I  - \boldsymbol{\Phi}^I \mathbf{A}  \|^2_F +\rho \| \boldsymbol{\Phi}^I \|_F\\
    (\boldsymbol{\Phi}^D)^* &= \min_{\boldsymbol{\Phi}^D} \| \mathbf{Y}^D  - \boldsymbol{\Phi}^D \mathbf{B}  \|^2_F +\rho \| \boldsymbol{\Phi}^D\|_F,
\end{align}
where $\|\cdot \|_F$ denotes the Frobenius norm, $\mathbf{Y}^I$, $\mathbf{Y}^D$, $\mathbf{A}$ and $\mathbf{B}$ are matrices whose columns are $\mathbf{y}^I_j$, $\mathbf{y}^D_j$, $\mathbf{a}_j$ and $\mathbf{b}_j$ respectively, and $j = 1,..., J$ indexes the signal examples from a given database. In addition to the reconstruction error, we have added a normalization constraint on the dictionaries, scaled by a small parameter $\rho$, in order to control the dictionary norms as usually done in dictionary learning. Before showing the performance of the proposed learning algorithm, we review prior art that we will use for experimental comparisons in Section~\ref{sec:results}.

\section{Relation to prior art}\label{sec:sota}

To the best of our knowledge, there has not been any work that addresses the problem of learning joint intensity-depth sparse representations. Therefore, we overview prior work that focuses on sparse approximation algorithms that bear similarities to the proposed JBP algorithm. Since the main characteristic of JBP is to find sparse approximations of two signals sharing a common sparse support, we overview algorithms targeting this problem. Such algorithms can be grouped into two categories with respect to the signal model they address: a) simultaneous sparse approximation algorithms, and b) group sparse approximation algorithms. We further discuss how algorithms from each group relate to JBP.


\textit{Simultaneous sparse approximation algorithms} recover a set of jointly sparse signals modeled as~\footnote{For the case of two signals, for example image intensity and depth, this model is a noisy version of the second model discussed in Sec.~\ref{sec:badmodels}.}:
  \begin{equation}
\label{eq:joint_sparse_support}
\mathbf{y}^s = \boldsymbol{\Phi} \mathbf{x}^s + \boldsymbol{\epsilon}^s= \sum_{i\in \mathcal{I}}\boldsymbol{\phi}_i x^s_i + \boldsymbol{\epsilon}^s, \hspace{0.5cm} s=1,...,S,
\end{equation}
where $S$ is the total number of signals $\mathbf{y}^s$, $\boldsymbol{\Phi}$ is the dictionary matrix and $\boldsymbol{\epsilon}^s$ is a noise vector for signal $\mathbf{y}^s$. Vectors of sparse coefficients $\mathbf{x}^s$ share the same sparsity support set $\mathcal{I}$, i.e., they have non-zero entries at the same positions. One of the earliest algorithms in this group is the Simultaneous Variable Selection (SVS) algorithm introduced by Turlach et. al.~\cite{tu:ve:wr:05}. SVS selects a common subset of atoms for a set of signals by minimizing the representation error while constraining the $\ell_{1}$-norm of the maximum absolute values of coefficients across signals. Formally, SVS solves the following problem:
\begin{align}
\text{(SVS)}:  \hspace{0.4 cm} &\min \frac{1}{2} \sum_{s=1}^S \| \mathbf{y}^s - \mathbf{\Phi}\mathbf{x}^s \|^2 , \hspace{0.4 cm} \\
\text{subject to:} \hspace{0.4 cm} 
&\sum_{i}{ \max \{ |x^{1}_i |,...,|x^S_i| \} } \leq \tau,
\end{align}
where $\tau$ is given. Let $\mathbf{X}$ denote the matrix with $\mathbf{x}^s$, $s=1,...,S$ as columns. We can see that the left hand side of the constraint in SVS is obtained by applying the $\ell_{\infty}$-norm to rows (to find the largest coefficients for all explanatory variables), followed by applying the $\ell_{1}$-norm to the obtained vector in order to promote sparsity of the support. We denote this norm as $\|\mathbf{X}\|_{\infty,1}$. Versions of the same problem for the unconstrained case and the error-constrained case have been studied by Tropp~\cite{Tropp2006589}.


To see the relation of SVS to JBP, we use Lemma~\ref{lm:bound_activity}, which allows us to formulate the JBP for the special case of $U^I=U^D$ as:
\begin{align}
\hspace{0.4 cm} &\min :  \hspace{0.4 cm} t  \hspace{0.4 cm} \\
\text{subject to:} \hspace{0.4 cm} 
& \| \mathbf{y}^D - \boldsymbol{\Phi}^D\mathbf{a} \|^2 \leq \epsilon^{2}\\
& \| \mathbf{y}^I - \boldsymbol{\Phi}^I\mathbf{b} \|^2 \leq \epsilon^{2}\\
& \sum_{i}{ \max \{ |a_{i}|, |b_{i}| \} }  \leq t . \label{eq:34lastcont}
\end{align}
Therefore, JBP operates on the same $\ell_{\infty,1}$-norm of the coefficient matrix as SVS. However, in contrast to SVS, JBP minimizes the number of non-zero elements in both $\mathbf{a}$ and $\mathbf{b}$ by minimizing $\| [\mathbf{a}\hspace{0.2cm} \mathbf{b}] \|_{\infty,1}$ and constraining the approximation error induced by the coefficients. A much more important difference of our work and~\cite{tu:ve:wr:05} is that we allow for different sets of atoms for intensity and depth. Thus, in JBP, each signal can be represented using a different dictionary, but with coefficient vectors that share the same positions of non-zero entries. This makes JBP applicable to intensity-depth learning, in contrast to SVS. Finally, we remark here that choosing the objective function as we did allows for a smooth convex representation of the last constraint (\ref{eq:34lastcont}).\\

\textit{Group sparse approximation algorithms} recover a signal modeled as:
    \begin{equation}
\label{eq:group_sparse_support}
\mathbf{y} = \sum_{i} \mathbf{H}_i \mathbf{x}_i + \boldsymbol{\epsilon}, 
\end{equation}
where $\mathbf{H}_i$ is a submatrix of a big dictionary matrix $\mathbf{H}$. This model is useful for signals whose sparse support has a group structure, namely when groups of coefficients are either all non-zero or all zero. 
The first algorithm proposed for group sparse approximation was a generalization of Lasso, developed by Bakin~\cite{Bakin99}, and later studied by other authors (e.g. Yuan and Lin~\cite{Yuan06}). Group Lasso refers to the following optimization problem:
\begin{align}
& \text{(GL)}:
 \min \; & \| \mathbf{y} - \sum_{i}  \mathbf{H}_i \mathbf{x}_i \|^2 + \lambda \sum_{i}\|\mathbf{x}_i \|_p,
\end{align}
where $\|\cdot \|_p$ denotes the $\ell_{p}$-norm. The most studied variant of group lasso is for $p=2$, because it leads to a convex optimization problem with efficient implementations. The group sparsity model can be used to represent intensity-depth signals by considering pairs $(a_i,b_i), i=1,...,N$ as groups. In this case, group lasso with $p=2$ becomes:
\begin{align}
\text{(GL-ID)}: 
 & \min ( \| \mathbf{y}^I - \sum_{i}  \boldsymbol{\phi}^I_i a_i \|^2 +\\
 &  \| \mathbf{y}^D - \sum_{i}  \boldsymbol{\phi}^D_i b_i \|^2 +  \lambda \sum_{i} \sqrt{a_{i}^2+b_{i}^2}).  \nonumber 
\end{align}
The drawback of GL with $p=2$ is that the square norm gives higher weight to balanced atom pairs (pairs with similar coefficients) than to asymmetric pairs with one large and one small coefficient. This means that GL would give priority to atom pairs with similar coefficients, which do not necessarily correspond to meaningful intensity depth pairs (see examples in Section~\ref{sec:badmodels}, where 3D features yield pairs with possibly large differences in coefficient values). Choosing $p=\infty$ avoids this problem and allows selection of pairs with unbalanced coefficients. In that case the regularizer penalizes the norm $\| [\mathbf{a}\hspace{0.2cm} \mathbf{b}] \|_{\infty,1}$. Rather than solving the unconstrained problem of group lasso with $p=\infty$ and a non-smooth objective, JBP reaches a similar goal by solving a constrained convex optimization problem with smooth constraints. It also eliminates the need for tuning the Lagrange multiplier.


 \section{Experimental results}\label{sec:results}
 
We have performed two sets of experiments in order to evaluate the proposed JBP and dictionary learning based on JBP. The first set of experiments uses simulated random data, with the goal to determine the model recovery performance of JBP when the ground truth signal models are given. In the second set, we apply JBP and dictionary learning on real depth-intensity data and show its performance on a depth inpainting task. In both cases, JBP has been compared to Group Lasso (GL). For the depth inpainting task, we also compare JBP to inpainting using total variation (TV)~\cite{Chambolle04}.

\subsection{Model recovery}\label{subsec:rec}
To evaluate the performance of JBP, we have generated a set of pairs of signals of size $N=64$, denoted by $\{\mathbf{y}^I_j\}$ and $\{\mathbf{y}^D_j\}$, $j=1,500$. Signals in each pair have a common sparsity support of size $|\mathcal{T}_0|$, and they are sparse in random, Gaussian iid dictionaries $\mathbf{\Phi}^I$ and $\mathbf{\Phi}^D$ of size $64\times128$. Their coefficients, $\{\mathbf{a}_j\}$ and $\{\mathbf{b}_j\}$, $j=1,500$ are random, uniformly distributed, and do not have the same values nor signs. However, their ratios $\alpha_i$ (as defined in Eq.~\ref{eq:alpha}) are bounded from below, which gives a certain value of $\gamma$ (see Eq.~\ref{eq:gamma}). Hence, we assume some similarity in the magnitudes within each pair of coefficients of the two modalities. All signals have been corrupted by Gaussian noise.

Figure~\ref{fig:JBPrecovery_coefs} shows the relative coefficient reconstruction error $\|\mathbf{a}^*-\mathbf{a}\|_2^2/ \|\mathbf{a}\|_2^2+ \|\mathbf{b}^*-\mathbf{b}\|_2^2/ \|\mathbf{b}\|_2^2$, where $(\mathbf{a}^*,\mathbf{b}^*)$ are the reconstructions of original values $(\mathbf{a},\mathbf{b})$. The error is averaged over 50 different signals and plotted versus the signal-to-noise (SNR) ratio between sparse signals and Gaussian noise. The parameter values for this evaluation set have been chosen as: $|\mathcal{T}_0| = 10$ and $\gamma=0.25$, which represent reasonable values that we would expect in real data. We have compared JBP with GL and with the theoretical bound in Eq.~\ref{eq:bound}, for $M=25$ and $M=64$. Instead of using the dictionary coherence value for $\delta$, which would give the worst-case bounds, we use the mean of inner products between all atoms to obtain and plot the average case bounds. We can see that JBP outperforms GL for a large margin. Moreover, the actual performance of JBP is much better than predicted by the theory, showing that the average derived bound is rather conservative.
\begin{figure}[!htbp]
\begin{center}
\includegraphics[width=0.49\textwidth]{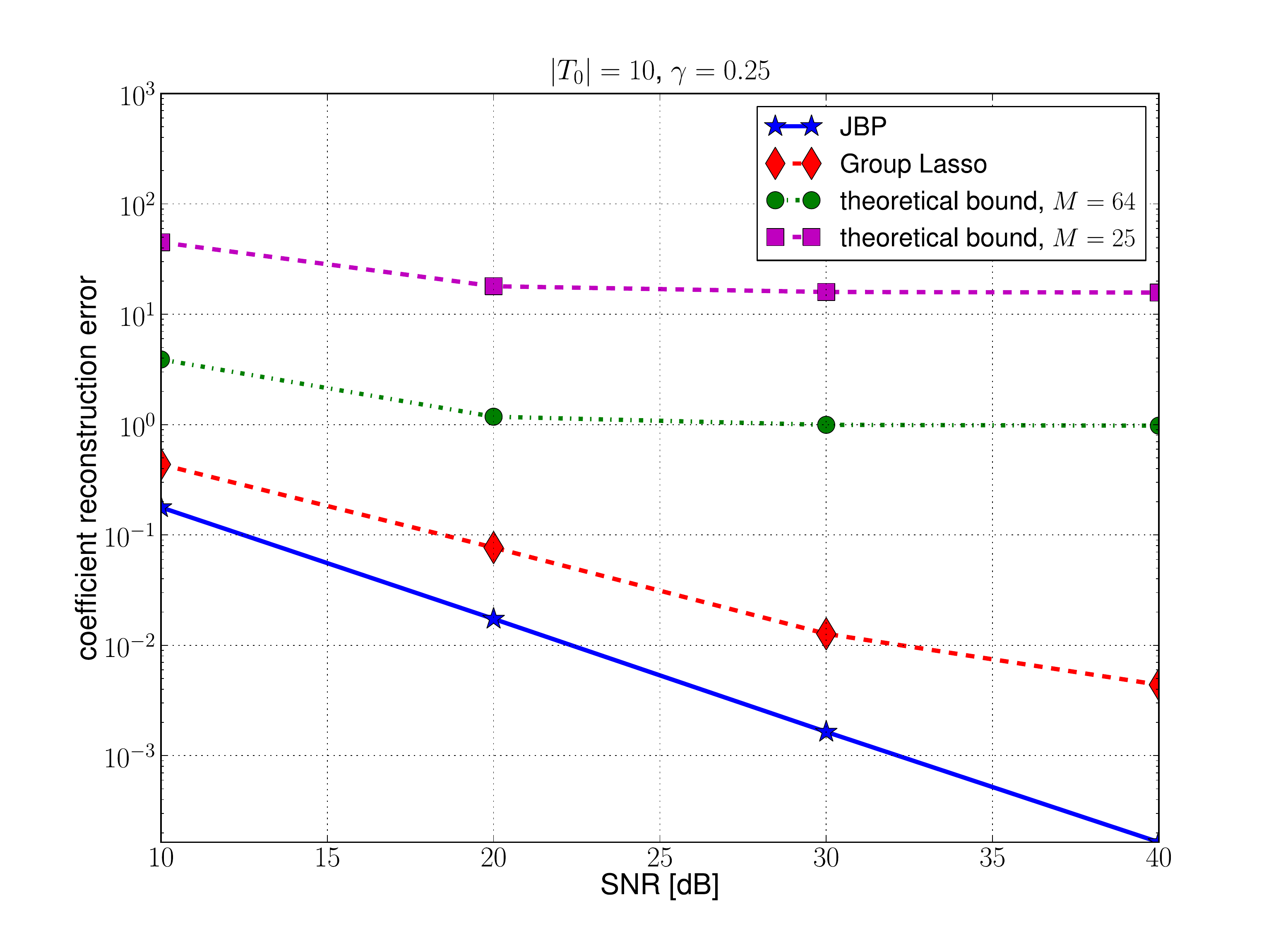}
\end{center}
\caption{JBP model recovery performance for random signals. Average relative coefficient reconstruction error is plotted for different signal-to-noise (SNR) ratios between sparse signals and Gaussian noise.}
\label{fig:JBPrecovery_coefs}
\end{figure}



Furthermore, we have used these randomly generated signals as training sets in our dictionary learning algorithm, in order to recover the original dictionary. For four different values of sparsity $|\mathcal{T}_0|=2,4,6,8$, we have applied the proposed learning algorithm starting from a random initial dictionary. For comparison, we have replaced the JBP in the inference step with GL, while keeping the learning step exactly the same. We refer to this method as GL-based learning.  Figure~\ref{fig:dico_recovery}a shows the mean square error (MSE) between the original atoms and the recovered ones vs sparsity $|\mathcal{T}_0|$, for JBP and GL-based learning. Similarly, we plot in Figure~\ref{fig:dico_recovery}b the percentage of recovered atoms vs sparsity, where an atom is considered recovered when its MSE is less than 0.05. Below this threshold the comparison is impossible since GL recovery error is huge (almost 0 recovered atoms). We can see from both graphs that learning based on JBP is superior to GL-based learning.

\begin{figure}[!htbp]
\begin{center}
\includegraphics[width=0.4\textwidth]{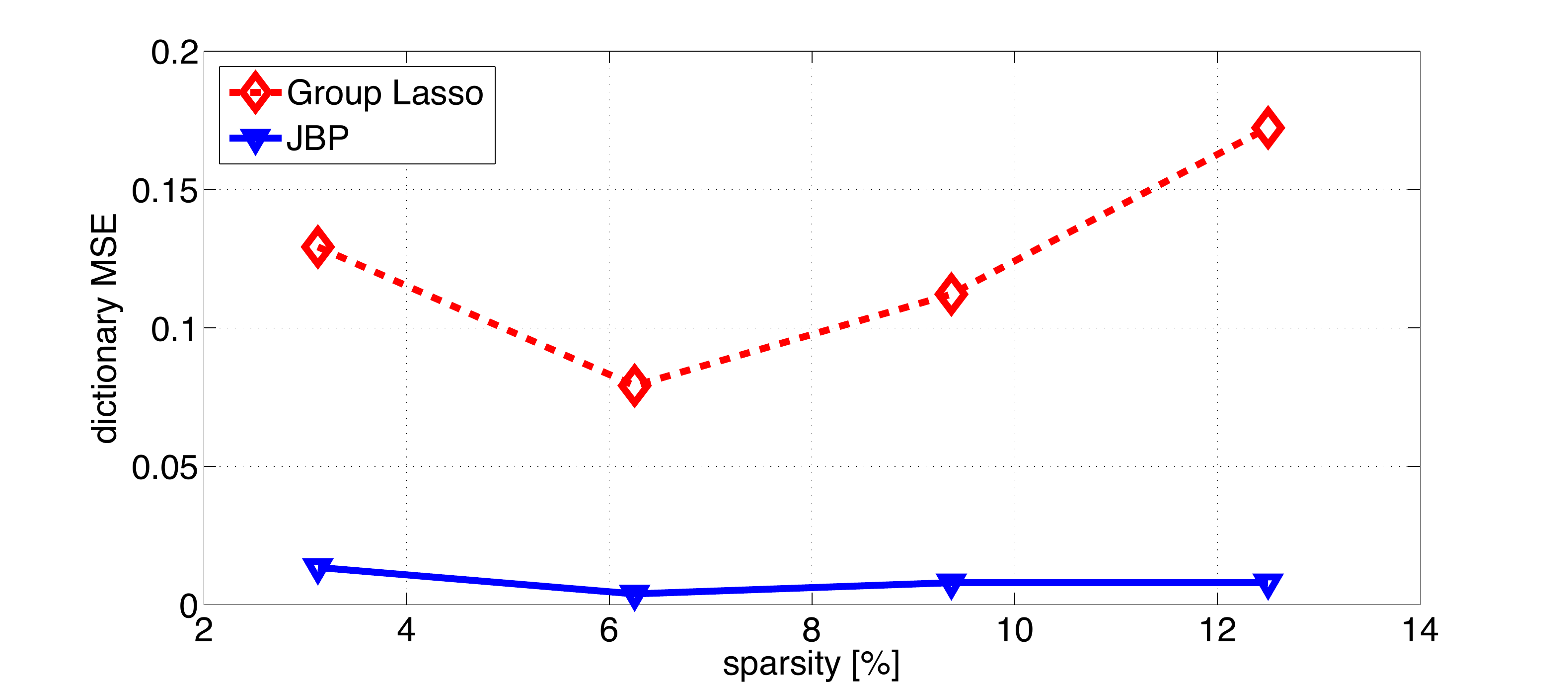}\\
\mbox{ \footnotesize{(a)}}\\
\includegraphics[width=0.4\textwidth]{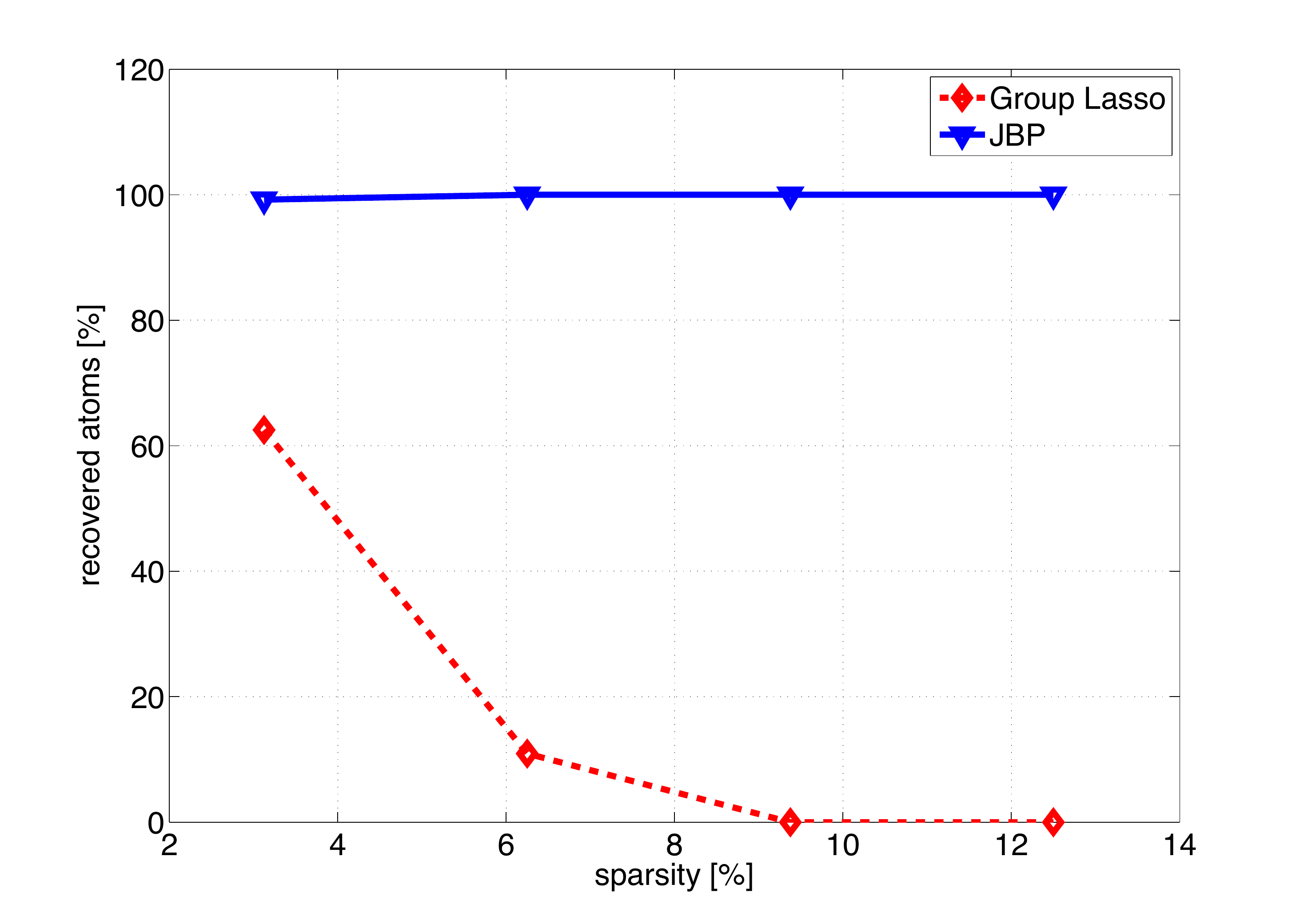}\\
\mbox{ \footnotesize{(b)}}\\
\end{center}
\caption{Recovery performance of dictionary learning using JBP and comparison to GL. (a) Mean square error between recovered atoms and original atoms vs sparsity $|\mathcal{T}_0|$. (b) Percentage of recovered atoms vs sparsity $|\mathcal{T}_0|$.}
\label{fig:dico_recovery}
\end{figure}


\subsection{Intensity-depth dictionary learning}\label{subsec:id}

 In our second set of experiments we have evaluated the performance of JBP and dictionary learning on real data, in particular on depth-intensity images. We have learned a depth-intensity overcomplete dictionary on the Middlebury 2006 benchmark depth-intensity data~\cite{middlebury}. The intensity data has been whitened, i.e., its frequency spectrum has been flattened, as initially proposed in~\cite{olshausen97sparse}. Such pre-processing speeds up the learning. Depth data could not be whitened because it would introduce Gibbs artifacts around the missing regions at occlusions. We handle such missing pixels by masking. Learning has been performed in a patch-mode. Namely, in each iteration of the two-step learning process, a large number of depth-intensity pairs of $12\times12$ size patches have been randomly selected from data. Each depth and intensity patch within a pair coincide in a 3D scene. Patches have been normalized to have norm one, and $\eta$ has been set to $0.1$. We have chosen this value such that we get a good reconstruction of depth, without the quantization effects present in Middlebury depth maps (i.e., such that the quantization error is subsumed by the reconstruction error). We have learned dictionaries $\mathbf{\Phi}^I$ and $\mathbf{\Phi}^D$, each of size $144\times288$, i.e., twice overcomplete. For comparison, we have also learned depth-intensity dictionaries using GL-based learning, where $\lambda=0.3$ has been chosen to obtain the same average reconstruction error as in JBP.
 
Figures~\ref{fig:dicos}a and Figures~\ref{fig:dicos}b show dictionaries learned by JBP and GL, respectively. The JBP-learned dictionary contains more meaningful features, such as coinciding depth-intensity edges, while GL-learned dictionary only has few of those. JBP dictionary atoms also exhibit correlation between orientations of the Gabor-like intensity atoms and the gradient angle of depth atoms. This is quite visible in the scatter plots of intensity orientation vs depth gradient angle shown in Figure~\ref{fig:scatter}. We can see that for JBP there is significant clustering around the diagonal (corresponding to a $90^{\circ}$ angle between orientation and gradient). On the other hand, we cannot see this effect when using GL for learning. To the best of our knowledge, this is the first time that the correlation between depth gradient angles and texture orientations is found to emerge from natural scenes data (see~\cite{PotetzLee2010} for some recent research in the area of 3D scene statistics).

Finally, we have compared the performance of JBP and GL, and the corresponding learned dictionaries, on an inpainting task. Namely, we have randomly removed 96\% of depth pixels from an intensity-depth pair obtained by a time-of-flight (TOF) camera\footnote{\url{http://www.pmdtec.com/}}. We have chosen the TOF data to show that learned dictionaries of intensity-depth are not linked to particular depth sensors. Original intensity and depth images are shown in Figures~\ref{fig:inpainting}a) and b), respectively. From the original intensity image and 4\% of depth pixels (shown in Figure~\ref{fig:inpainting}c), we have reconstructed the whole depth image, using GL with the GL-learned dictionary (Figure~\ref{fig:inpainting}d), and using JBP with the JBP-learned dictionary (Figure~\ref{fig:inpainting}e). We have also applied TV inpainting on depth masked image only and obtained the result shown in Figure~\ref{fig:inpainting}f. We can see that JBP gives the best performance (mean square error MSE=4.9e-3), followed by GL (MSE=7.2e-3) and TV (MSE=7.7e-3). Therefore, GL gives just a minor improvement to TV inpainting (which does not use the intensity image), while JBP gives a significantly smaller MSE compared to both GL and TV.

 
\begin{figure*}[!htbp]
\begin{center}
\includegraphics[width=0.9\textwidth]{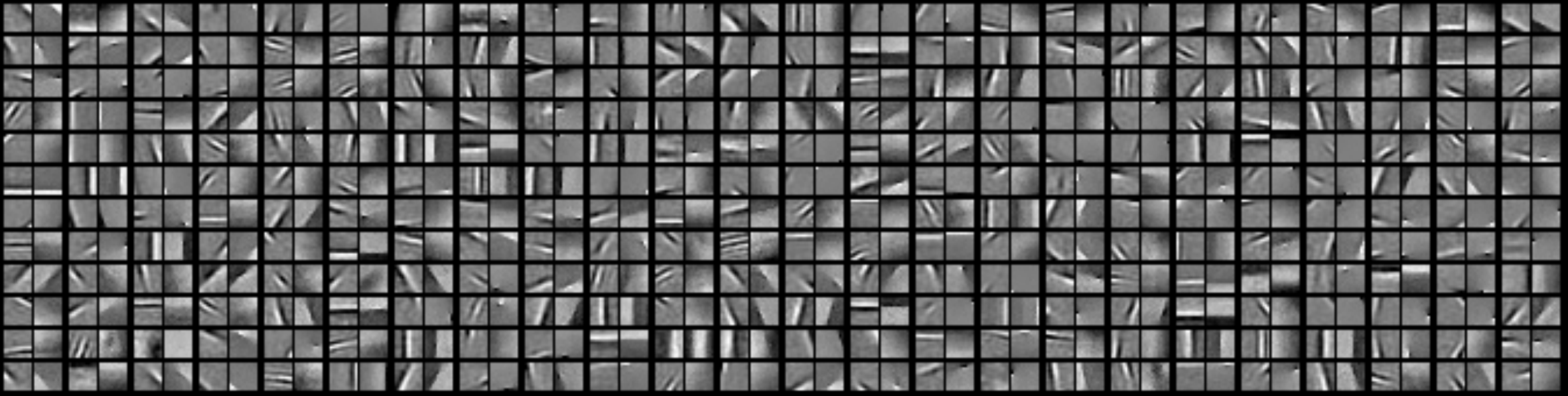}\\
\mbox{ (a) JBP}\\
~\\
\includegraphics[width=0.9\textwidth]{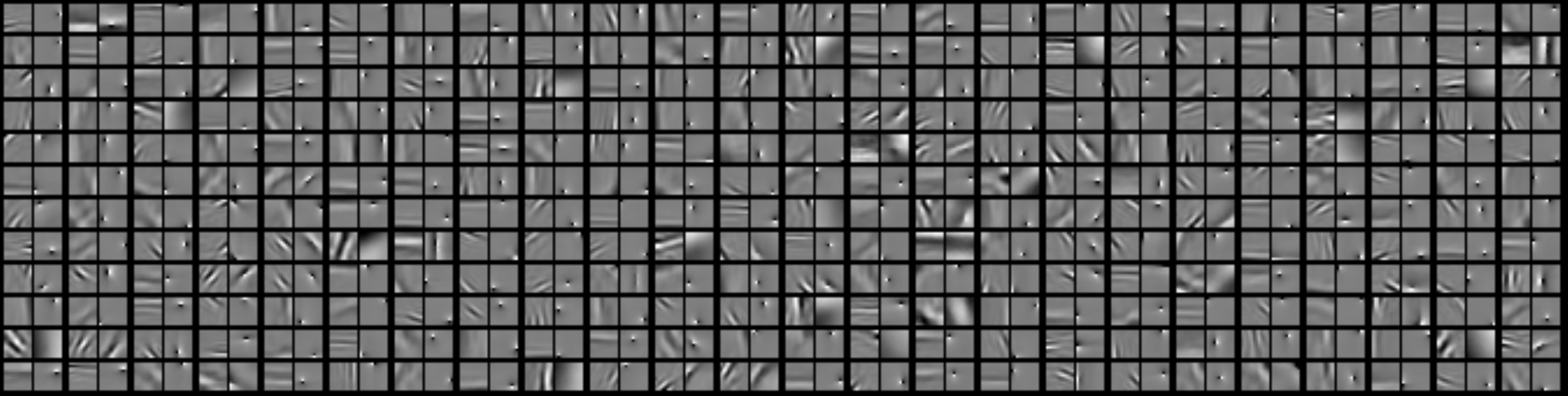}\\
\mbox{ (b) GL}\\
\end{center}
\caption{Learned intensity-depth dictionaries. Each column contains a set of atom pairs ($\boldsymbol{\phi}^I_1,\boldsymbol{\phi}^D$), where the left part is an intensity atom and the right part is a depth atom. (a) JBP-learned dictionaries, (b) GL-learned dictionaries.}
\label{fig:dicos}
\end{figure*}

\begin{figure*}[!htbp]
\begin{center}
$\begin{array}{c@{\hspace{1 cm}}c@{\hspace{1 cm}}c}
\includegraphics[width=0.15\textwidth]{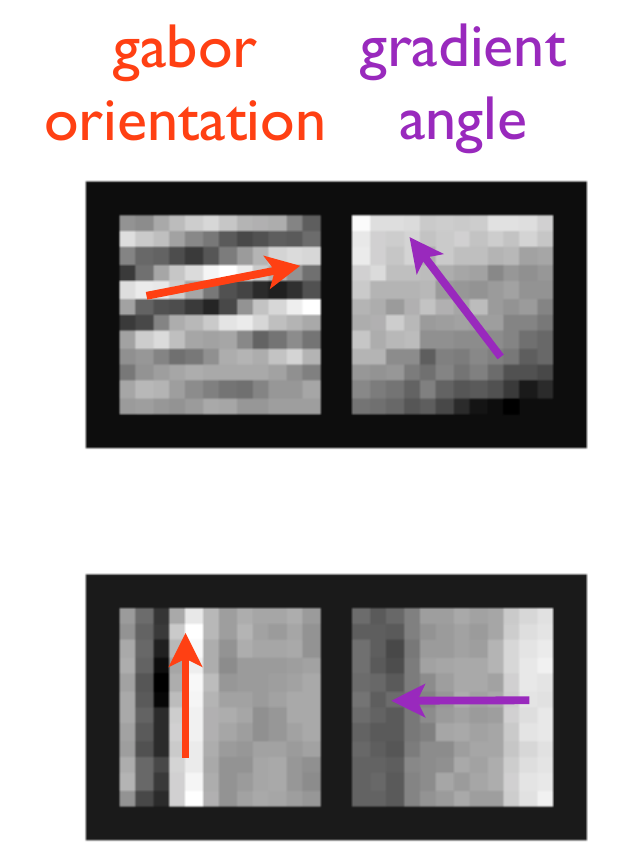}&
\includegraphics[width=0.3\textwidth]{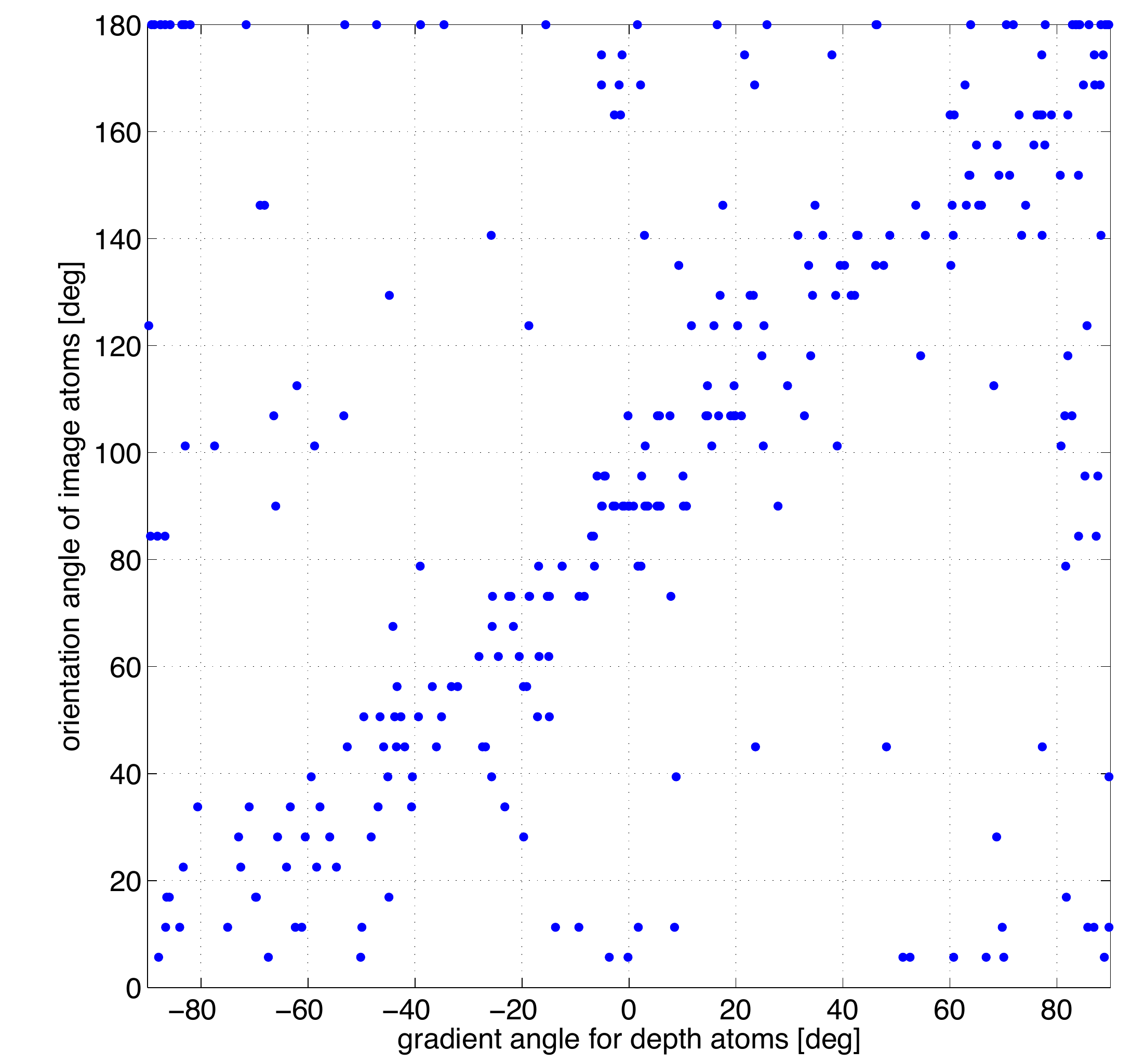}&
\includegraphics[width=0.3\textwidth]{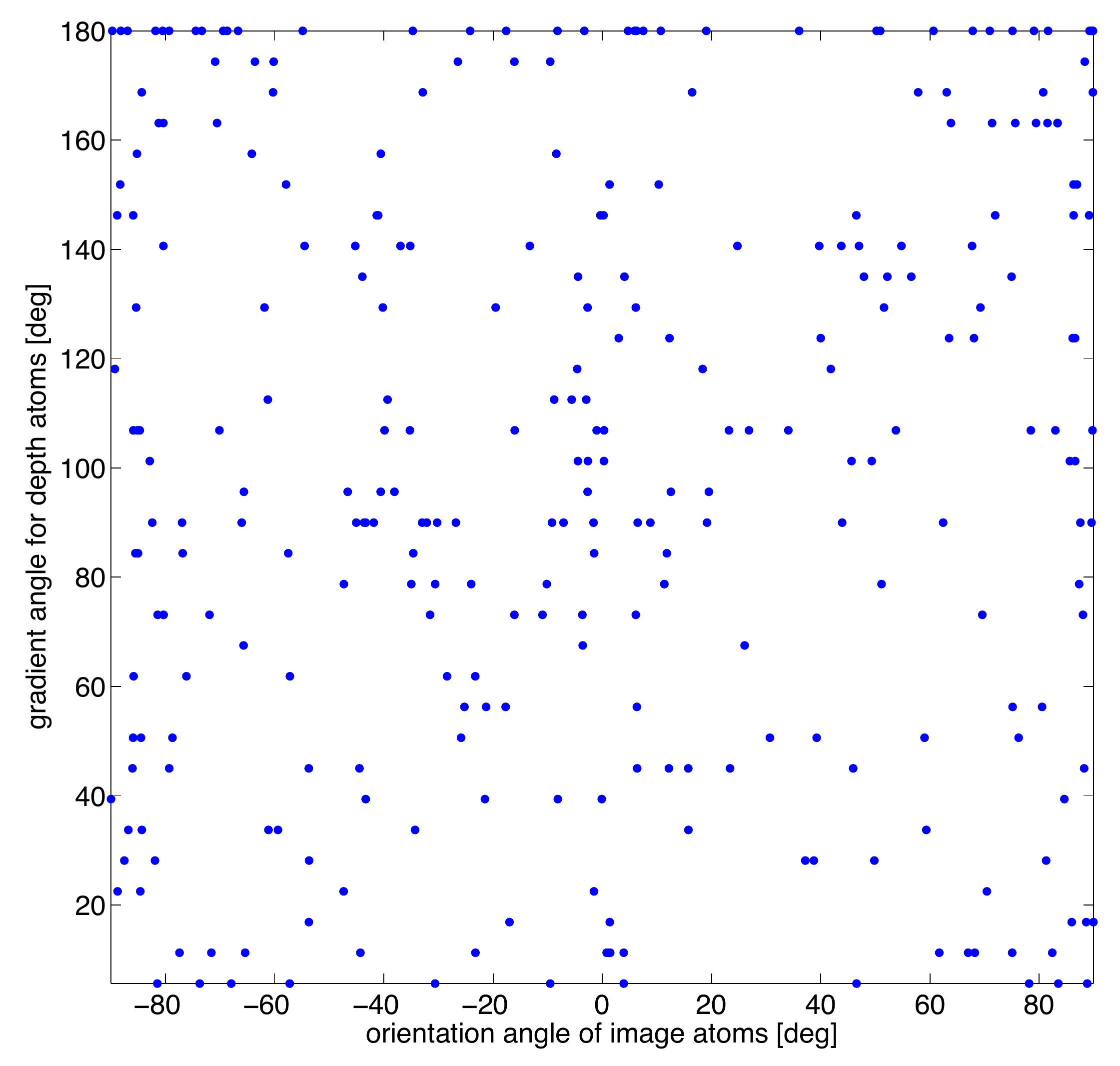}\\
\mbox{ \footnotesize{(a)} } &\mbox{ \footnotesize{(b)} } &\mbox{ \footnotesize{(c)} } \\
\end{array}$
\end{center}
\caption{Correlation between depth atom gradients and image atom orientations. a) Illustration of atom pairs that have 90 degrees angle between the orientation of the Gabor-like intensity part and the gradient angle of the depth part. Scatter plots of orientation vs gradient angle for b) JBP and c) GL.}

\label{fig:scatter}
\end{figure*}

   \begin{figure*}[!htbp]
\begin{center}
$\begin{array}{c@{\hspace{0.05 cm}}c@{\hspace{0.05 cm}}c@{\hspace{0.05 cm}}c@{\hspace{0.05 cm}}c@{\hspace{0.05 cm}}c}
\includegraphics[width=0.16\textwidth]{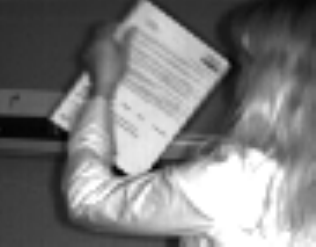}&
\includegraphics[width=0.16\textwidth]{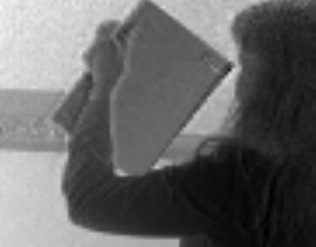}&
\includegraphics[width=0.16\textwidth]{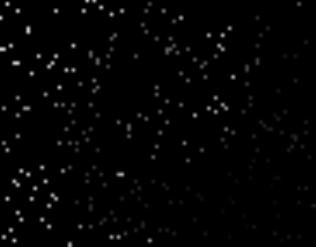}&
\includegraphics[width=0.16\textwidth]{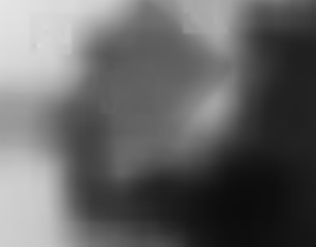}&
\includegraphics[width=0.16\textwidth]{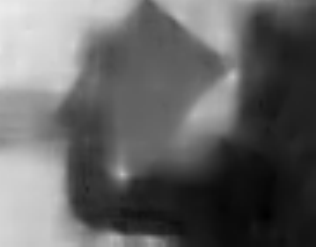}&
\includegraphics[width=0.16\textwidth]{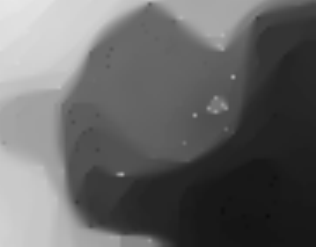}\\
\mbox{ \footnotesize{(a)} } &\mbox{ \footnotesize{(b)} } &\mbox{ \footnotesize{(c)} } &\mbox{ \footnotesize{(d) GL} }&\mbox{ \footnotesize{(e) JBP}  }&\mbox{ (f) \footnotesize{TV} }\\
\end{array}$
\end{center}
\caption{Inpainting results on time of flight data. a) Original intensity image, b) Original depth image, c) 4\% of kept depth pixels, d) reconstructed depth with GL; MSE = 7.2e-3, e) reconstructed depth with JBP, MSE = 4.9e-3, f) reconstructed depth with total variation inpainting, MSE=7.7e-3.}
\label{fig:inpainting}
\end{figure*}

\section{Conclusion}
We have presented an algorithm for learning joint overcomplete dictionaries of image intensity and depth. The proposed method is based on a novel second order cone program (called JBP) for recovering sparse signals of joint sparse support in dictionaries with two modalities. We have derived a theoretical bound for the coefficient recovery error of JBP and shown its superiority to the Group Lasso algorithm through numerical simulations. When applied to the Middlebury image-depth database, the proposed learning algorithm converges to a dictionary of various intensity-depth features, such as coinciding edges and image grating - depth slant pairs. The learned features exhibit a significant correlation of depth gradient angles and texture orientations, which is an important result in 3D scene statistics research. Finally, we have shown that JBP with the learned dictionary can reconstruct meaningful depth maps from only 4\% of depth pixels. These results outline the important value of our method for 3D technologies based on hybrid image-depth sensors.

\begin{appendix}\label{app:proof}
\subsection{Proof of Theorem~\ref{thm:bound}}

Let us first prove the following lemma:
\begin{lemma}
For $\mathbf{h}:= [\mathbf{a}^{*}; \mathbf{b}^{*}] - [\mathbf{a}^{0}; \mathbf{b}^{0} ] $ it holds true that:
\begin{equation}
\label{eq:cone_constraint}
\|\mathbf{h}_{\mathcal{T}_0^C} \|_1 \leq \|\mathbf{h}_{\mathcal{T}_0}\|_1 + \gamma U |\mathcal{T}_0|,
\end{equation}
where $\mathcal{T}_0^C$ denotes the complement set of $\mathcal{T}_0$ and $\mathbf{h}_{\mathcal{T}}$ denotes the subvector of $\mathbf{h}$ corresponding to $\mathcal{T}$.
\end{lemma}
\begin{proof}
Define
\begin{align*}
\mathcal{I}_{a}^{0}:= &  \{i \in \mathcal{I}: |a_{i}^{0}| = U x_i{^{0}} \}, \\
\mathcal{I}_{b}^{0}:= &  \{i \in \mathcal{I} \setminus \mathcal{I}_{a}^{0}: |b_{i}^{0}| = U x_i{^{0}} \}, \\
\mathcal{I}_{a}^{*}:= &  \{i \in \mathcal{I}: |a_{i}^{*}| = U x_i{^{*}} \}, \\
\mathcal{I}_{b}^{*}:= &  \{i \in \mathcal{I}\setminus \mathcal{I}_{a}^{*}: |b_{i}^{*}| = U x_i{^{*}} \}. \\
\end{align*}
Due to Lemma~\ref{lm:bound_activity}, we have that $\mathcal{I}_{a}^{0} \cup \mathcal{I}_{b}^{0} = \mathcal{I}$ and $\mathcal{I}_{a}^{*} \cup \mathcal{I}_{b}^{*} = \mathcal{I}$, and due to the definition above it holds that $\mathcal{I}_{a}^{0} \cap \mathcal{I}_{b}^{0} = \emptyset$ and $\mathcal{I}_{a}^{*} \cap \mathcal{I}_{b}^{*} = \emptyset$. Therefore, we have that:
\begin{align}
\label{eq:cone_right}
\| [\mathbf{a}^*;\mathbf{b}^*]\|_1 &= \sum_{i \in \mathcal{I}_{a}^{*}} |a_i^*| + \sum_{i \in \mathcal{I}_{b}^{*}} |b_i^*|+ \sum_{i \in \mathcal{I}_{a}^{*}} |b_i^*| + \sum_{i \in \mathcal{I}_{b}^{*}} |a_i^*| \nonumber \\
&\leq U \sum_{i \in \mathcal{I}} |x_i^*| + U \sum_{i \in \mathcal{I}_{a}^{*}} |x_i^*| + U \sum_{i \in \mathcal{I}_{b}^{*}}  |x_i^*| \nonumber\\
&= 2U\|\mathbf{x}^*\|_1.
\end{align}
Similarly, we have that:
\begin{align}
\label{eq:cone_left}
\| [\mathbf{a}^0;\mathbf{b}^0]\|_1 &= \sum_{i \in \mathcal{I}_{a}^{0}} |a_i^0| + \sum_{i \in \mathcal{I}_{b}^{0}} |b_i^0|+ \sum_{i \in \mathcal{I}_{a}^{0}} |b_i^0| + \sum_{i \in \mathcal{I}_{b}^{0}} |a_i^0| \nonumber \\
& \geq U \sum_{i \in \mathcal{I}} |x_i^0| + \min_{i \in \mathcal{T}_{0}}\alpha_i (\sum_{i \in \mathcal{I}_{a}^{0}}  |a_i^0| + \sum_{i \in \mathcal{I}_{b}^{0}}  |b_i^0| ) \nonumber  \\
&\geq^{\textnormal{\eqref{eq:gamma}}}  2U\|\mathbf{x}^0\|_1 - \gamma U |\mathcal{T}_0|.
\end{align}
Due to optimality of $\mathbf{x}^{*}$, we have $\| \mathbf{x}^{*} \|_{1} \leq \| \mathbf{x}^{0} \|_{1}$, which combined with (\ref{eq:cone_right}) and (\ref{eq:cone_left}) gives:
\begin{align}\label{eq:norm_relation}
\| [\mathbf{a}^*;\mathbf{b}^*]\|_1 \leq 2U\|\mathbf{x}^0\|_1 \leq \| [\mathbf{a}^0;\mathbf{b}^0]\|_1 + \gamma U |\mathcal{T}_0|.
\end{align}
Due to $\mathbf{a}^{0}_{\mathcal{T}_{0}^{C}} = \mathbf{0}$ and $\mathbf{b}^{0}_{\mathcal{T}_{0}^{C}} = \mathbf{0}$, we can write
\begin{align}
\| [\mathbf{a}^{0}; \mathbf{b}^{0}] + \mathbf{h} \|_{1} &= \|[\mathbf{a}_{\mathcal{T}_{0}}^{0}; \mathbf{b}_{\mathcal{T}_{0}}^{0}; \mathbf{0}] + [\mathbf{h}_{\mathcal{T}_{0}}; \mathbf{h}_{\mathcal{T}^{C}_{0}}] \|_{1} \nonumber\\
&= \| [\mathbf{a}_{\mathcal{T}_{0}}^{0}; \mathbf{b}_{\mathcal{T}_{0}}^{0}] + \mathbf{h}_{\mathcal{T}_{0}}\|_{1} + \| \mathbf{h}_{\mathcal{T}^{C}_{0}} \|_{1}.
\end{align}
Thus, using the triangle inequality and the definition of $\mathbf{h}$ we derive:
\begin{align*}
& \| [\mathbf{a}^{0}; \mathbf{b}^{0}]\|_{1} -  \| \mathbf{h}_{\mathcal{T}_0} \|_{1} + \| \mathbf{h}_{\mathcal{T}_{0}^{C}} \|_{1} \leq  \| [\mathbf{a}^{0}; \mathbf{b}^{0}] + \mathbf{h} \|_{1} \nonumber \\
&= \| [\mathbf{a}^{*}; \mathbf{b}^{*}]\|_{1} \leq^{\textnormal{\eqref{eq:norm_relation}}}  \| [\mathbf{a}^{0}; \mathbf{b}^{0}]\|_{1}  + \gamma U |\mathcal{T}_0|
\end{align*}
and thus
\begin{equation}
\label{eq:coneconstr}
\| \mathbf{h}_{\mathcal{T}_{0}^{C}} \|_{1} \leq \| \mathbf{h}_{\mathcal{T}_{0}} \|_{1} +  \gamma U |\mathcal{T}_0|.
\end{equation}

\end{proof}

We are now ready to prove Theorem~\ref{thm:bound}.
\begin{proof}
Let $\mathbf{A}$ be defined as in Eq.~(\ref{eq:matA}). Then we have from \eqref{o2:11} and \eqref{o2:12}  that
\[
\| \mathbf{A} \mathbf{h} \|_{2} \leq 4 \epsilon = 4\eta f_0.
\]
Assume we have divided $\mathcal{T}^{C}_{0}$ into subsets of size $M$, more precisely, we have
$ \mathcal{T}^{C}_{0} = \mathcal{T}_{1} \cup \dots \cup \mathcal{T}_{n-|\mathcal{T}_{0}|}$, where $\mathcal{T}_{i}$ are sorted by decreasing order of $\mathbf{h}_{\mathcal{T}^{C}_{0}}$, and where $\mathcal{T}_{01} = \mathcal{T}_{0} \cup \mathcal{T}_{1}$. Without alternations - cf.  \cite{ca:ro:te:05} - it holds true that 
\[
\| \mathbf{h}_{\mathcal{T}_{01}^{C}} \|^{2}_{2} \leq \| \mathbf{h}_{\mathcal{T}^{C}_{0}} \|_{1}^2 / M.
 \]
 Using \eqref{eq:coneconstr} yields now
 \begin{align}
\| \mathbf{h}_{\mathcal{T}_{01}^{C}} \|^{2}_{2} &\leq (\| \mathbf{h}_{\mathcal{T}_{0}} \|_{1} +  \gamma U |\mathcal{T}_0|)^{2} / M \nonumber\\
&\leq (\sqrt{|\mathcal{T}_{0}|}\| \mathbf{h}_{\mathcal{T}_{0}} \|_{2} +  \gamma U |\mathcal{T}_0|)^{2} / M,
 \end{align}
 where the second step follows from the norm inequality. Hence:
 \begin{align}\label{eq:ernorm}
 \| \mathbf{h} \|^{2}_{2}  &= \| \mathbf{h}_{\mathcal{T}_{01}} \|^{2}_{2} + \| \mathbf{h}_{\mathcal{T}_{01}^{C}} \|^{2}_{2} \nonumber\\
 & \leq ( 1 +  \frac{ | \mathcal{T}_{0}|}{ M} ) \| \mathbf{h}_{\mathcal{T}_{0}} \|_{2}^{2} + \frac{2 \gamma U | \mathcal{T}_{0}|^{3/2}}{ M}   \| \mathbf{h}_{\mathcal{T}_{0}} \|_{2}  \nonumber\\
 &+ \frac{(\gamma U |\mathcal{T}_0|)^{2}}{M}.
 \end{align}
 From the restricted isometry hypothesis, cf. Def. \ref{def:isometry}, we get
 \begin{align}  
 \label{eq:first}
\| \mathbf{A} \mathbf{h} \|_{2} &= \| \mathbf{A}_{\mathcal{T}_{01}} \mathbf{h}_{\mathcal{T}_{01}} + \sum_{j \geq 2}  \mathbf{A}_{\mathcal{T}_{j}} \mathbf{h}_{\mathcal{T}_{j}} \|_{2}\nonumber\\
& \geq  \| \mathbf{A}_{\mathcal{T}_{01}} \mathbf{h}_{\mathcal{T}_{01}} \|_{2} - \| \sum_{j \geq 2}  \mathbf{A}_{\mathcal{T}_{j}} \mathbf{h}_{\mathcal{T}_{j}} \|_{2} \nonumber\\ 
&\geq  \| \mathbf{A}_{\mathcal{T}_{01}} \mathbf{h}_{\mathcal{T}_{01}} \|_{2} -  \sum_{j \geq 2} \| \mathbf{A}_{\mathcal{T}_{j}} \mathbf{h}_{\mathcal{T}_{j}} \|_{2} \nonumber\\
&\geq \sqrt{1-\delta_{M+|\mathcal{T}_{0}|}} \| \mathbf{h}_{\mathcal{T}_{01}} \|_{2} -  \sqrt{1+\delta_{M}} \sum_{j \geq 2} \|  \mathbf{h}_{\mathcal{T}_{j}} \|_{2} \nonumber\\
&\geq \sqrt{1-\delta_{M+|\mathcal{T}_{0}|}} \| \mathbf{h}_{\mathcal{T}_{0}} \|_{2} -  \sqrt{1+\delta_{M}} \sum_{j \geq 2} \|  \mathbf{h}_{\mathcal{T}_{j}} \|_{2} 
\end{align}
where $\delta_{S} $ is a constant chosen such that the inequalities hold, which follows from inequality 
 (4) in \cite{ca:ro:te:05}. Here, $\mathbf{A}_\mathcal{T}$ denotes the columns of $\mathbf{A}$ corresponding to the index set $\mathcal{T}$.\\

 In analogy to \cite{ca:ro:te:05}, due to the ordering of the sets $\mathcal{T}_{j}$ by decreasing order of coefficients, we have:
 \[
 |\mathbf{h}_{\mathcal{T}_{j+1}(t)}| \leq  \| \mathbf{h}_{\mathcal{T}_{j}}\|_{1} / M 
 \]
 meaning each component in $\mathbf{h}_{\mathcal{T}_{j+1}}$ is smaller than the average of the components in $\mathbf{h}_{\mathcal{T}_{j}}$ (absolute value-wise). Thus, we get:
 \begin{align*}
  \| \mathbf{h}_{\mathcal{T}_{j+1}} \|^{2}_{2} & = \sum_{t \in \mathcal{T}_{j+1} } \|  \mathbf{h}_t  \|^{2}_{2} \\
  & \leq  \sum_{t \in \mathcal{T}_{j+1} }   \| \mathbf{h}_{\mathcal{T}_{j}}\|^{2}_{1} / M^{2} \\ 
   & \leq  M \| \mathbf{h}_{\mathcal{T}_{j}}\|^{2}_{1} / M^{2} = \| \mathbf{h}_{\mathcal{T}_{j}}\|^{2}_{1} / M,
\end{align*}
and
 \begin{align}
 \label{eq:sec}
  \sum_{j \geq 2 }   \| \mathbf{h}_{\mathcal{T}_{j}}\|_{2} & \leq \sum_{j \geq 1 }   \| \mathbf{h}_{\mathcal{T}_{j}}\|_{1} / \sqrt{M}\nonumber\\
 & =  \| \mathbf{h}_{\mathcal{T}_{0}^{C}} \|_{1} / \sqrt{M} \nonumber\\
 & \leq^{\textnormal{\eqref{eq:coneconstr}}} (\| \mathbf{h}_{\mathcal{T}_{0}} \|_1 + \gamma U |\mathcal{T}_0|)/\sqrt{M}\nonumber\\
 & \leq  \sqrt{{| \mathcal{T}_0 |/ M}} \| \mathbf{h}_{\mathcal{T}_{0}} \|_2 + \gamma U |\mathcal{T}_0|/\sqrt{M}
\end{align}
where the last step follows from the norm inequality. Combining Eq.~(\ref{eq:sec}) and Eq.~(\ref{eq:first}), we get:
 \begin{align}  
 \label{eq:fin}
\| A \mathbf{h} \|_{2} \geq & \sqrt{1-\delta_{M+|\mathcal{T}_{0}|}} \| \mathbf{h}_{\mathcal{T}_{0}} \|_{2} \nonumber\\
&-  \sqrt{1+\delta_{M}}  \sqrt{{| \mathcal{T}_0 |/ M}} \| \mathbf{h}_{\mathcal{T}_{0}} \|_2\nonumber\\
& -  \gamma U |\mathcal{T}_0| \sqrt{1+\delta_{M}} /\sqrt{M}
\end{align}
and subsequently:
 \begin{align}  
 \label{eq:fin}
\| \mathbf{h}_{\mathcal{T}_{0}} \|_{2} \leq & \frac{\| A\mathbf{h} \|_2+  \gamma U |\mathcal{T}_0| \sqrt{1+\delta_{M}} /\sqrt{M}}{\sqrt{1-\delta_{M+|\mathcal{T}_{0}|}}  -  \sqrt{1+\delta_{M}}  \sqrt{{| \mathcal{T}_0 |/ M}} } \\
\leq & \frac{4\eta f_0 \sqrt{M}+ \gamma f_0 |\mathcal{T}_0| \sqrt{1+\delta_{M}} }{\sqrt{M(1-\delta_{M+|\mathcal{T}_{0}|})}  -  \sqrt{|\mathcal{T}_0|(1+\delta_{M})}} \nonumber\\
& = C f_0,
\end{align}
if the denominator is greater than zero. Replacing this result in Eq.~(\ref{eq:ernorm}) and taking $U=f_0$ we get:
 \begin{equation}
 \| \mathbf{h} \|^{2}_{2}  \leq ( 1 +  \frac{ | \mathcal{T}_{0}|}{ M} ) C^{2}f_0^2 + 2\gamma \frac{ | \mathcal{T}_{0}|^{3/2}}{ M}  C f_0^2  + \gamma^2 \frac{ |\mathcal{T}_0|^{2}}{M} f_0^2,
 \end{equation}
 which is equivalent to~(\ref{eq:bound}) and thus completes the proof.
   \end{proof}

\end{appendix}

\bibliographystyle{IEEEtran}
\bibliography{int_depth}
\end{document}